\documentclass[runningheads]{llncs}
\usepackage{graphicx}  
\usepackage{wrapfig}
\usepackage[caption=false]{subfig}
\usepackage{pgfplots}

\usepackage{todonotes}

\usepackage[inline]{enumitem}

\newtheorem{cor}{Corollary}
\usepackage[modernsign,substopindex,shortmquant,mquantifiertype,mconnectiveformal,bracketinterpret,postfixinterpret,bracketmodalinterpret,setfixinterpret,modifopindex,seqarrow,seqoptional,sidenotecalculus,abbrseqcontext,shortterms,nosigmaterms,novarterms]{logic}
\usepackage[bracketinterpret,postfixinterpret,bracketmodalinterpret]{dL}
\usepackage{amsmath}

\usepackage{hyperref}
\usepackage{stmaryrd}

\usepackage{prettyref}
\newcommand{\rref}[2][]{\prettyref{#2}}
\newrefformat{model}{Model\,\ref{#1}}
\newrefformat{listing}{Listing\,\ref{#1}}
\newrefformat{line}{line\,\ref{#1}}
\newrefformat{sec}{Section\,\ref{#1}}
\newrefformat{appendix}{Appendix\,\ref{#1}}
\newrefformat{def}{Definition\,\ref{#1}}
\newrefformat{thm}{Theorem\,\ref{#1}}
\newrefformat{ax}{\ref{#1}}
\newrefformat{prop}{Proposition\,\ref{#1}}
\newrefformat{lemma}{Lemma\,\ref{#1}}
\newrefformat{cor}{Corollary\,\ref{#1}}
\newrefformat{ex}{Example\,\ref{#1}}
\newrefformat{tab}{Table\,\ref{#1}}
\newrefformat{fig}{Figure\,\ref{#1}}
\newrefformat{eqn}{(\ref{#1})}

\usepackage{xcolor, colortbl}
\usepackage{color}
\definecolor{gray}{rgb}{0.5,0.5,0.5}
\definecolor{ForestGreen}{HTML}{228b22}

\newcommand{\tvar}[1]{\texttt{#1}}
\newcommand{\den}[1]{\llbracket#1\rrbracket}

\usepackage{listings}
\usepackage{verbatim}
\usepackage{courier}
\usepackage{upquote}

\usepackage[firstpage]{draftwatermark}
\SetWatermarkText{\hspace*{4.5in}\raisebox{6.3in}{\includegraphics[scale=0.1]{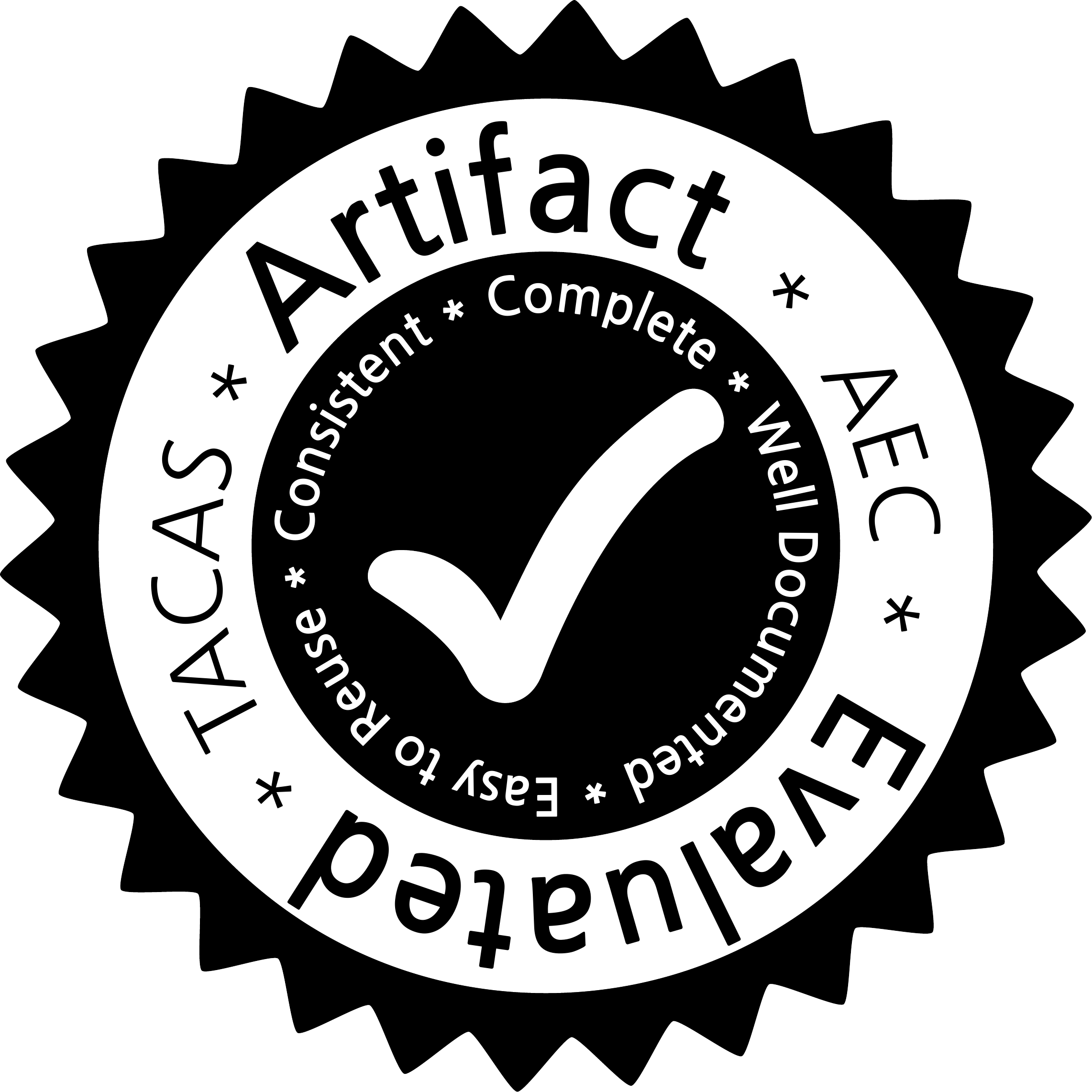}}}
\SetWatermarkAngle{0}

\usepackage{environ}
\NewEnviron{killcontents}{}

\begin{document}

\title{Verifiably Safe Off-Model\\ Reinforcement Learning\thanks{%
This research was sponsored by the Defense Advanced Research Projects Agency (DARPA) under grant number FA8750-18-C-0092.}}
\author{Nathan Fulton\orcidID{0000-0002-4172-7631} \and Andr\'{e} Platzer\orcidID{0000-0001-7238-5710}}

\institute{Computer Science Department, Carnegie Mellon University \\ 
Pittsburgh, USA \\
 \email{\{nathanfu, aplatzer\}@cs.cmu.edu}
}

\maketitle

\begin{abstract}
The desire to use reinforcement learning in safety-critical settings has inspired a recent interest in formal methods for learning algorithms. Existing formal methods for learning and optimization primarily consider the problem of constrained learning or constrained optimization. Given a single correct model and associated safety constraint, these approaches guarantee efficient learning while provably avoiding behaviors outside the safety constraint. Acting well given an accurate environmental model is an important pre-requisite for safe learning, but is ultimately insufficient for systems that operate in complex heterogeneous environments. This paper introduces verification-preserving model updates, the first approach toward obtaining formal safety guarantees for reinforcement learning in settings where multiple environmental models must be taken into account. Through a combination of design-time model updates and runtime model falsification, we provide a first approach toward obtaining formal safety proofs for autonomous systems acting in heterogeneous environments.
\end{abstract}

\section{Introduction}
\label{sec:introduction}

The desire to use reinforcement learning in safety-critical settings has inspired several recent approaches toward obtaining formal safety guarantees for learning algorithms.
Formal methods are particularly desirable in settings such as self-driving cars, where testing alone cannot guarantee safety \cite{RANDDriveToSafety}.
Recent examples of work on formal methods for reinforcement learning algorithms include justified speculative control \cite{aaai18}, shielding \cite{DBLP:conf/aaai/AlshiekhBEKNT18}, logically constrained learning \cite{HasanbeigKroening},
and constrained Bayesian optimization \cite{DBLP:journals/corr/ghoshetal}. 
Each of these approaches provide formal safety guarantees for reinforcement learning and/or optimization algorithms
by stating assumptions and specifications in a formal logic, 
generating monitoring conditions based upon specifications and environmental assumptions, 
and then leveraging these monitoring conditions to constrain the learning/optimization process to a known-safe subset of the state space. 

Existing formal methods for learning and optimization consider the problem of constrained learning or constrained optimization \cite{DBLP:conf/aaai/AlshiekhBEKNT18,aaai18,DBLP:journals/corr/ghoshetal,HasanbeigKroening}.
They address the question: 
assuming we have a single accurate environmental model with a given specification, how can we learn an efficient control policy respecting this specification? 

Correctness proofs for control software in a single well-modeled environment are necessary but not sufficient for ensuring that reinforcement learning algorithms behave safely. Modern cyber-physical systems must perform a large number of subtasks in many different environments and must safely cope with situations that are not anticipated by system designers. These design goals motivate the use of reinforcement learning in safety-critical systems. Although some formal methods suggest ways in which formal constraints might be used to inform control even when modeling assumptions are violated \cite{aaai18}, none of these approaches provide formal safety guarantees when environmental modeling assumptions are violated.

Holistic approaches toward safe reinforcement learning should provide formal guarantees 
even when a single, a priori model is not known at design time.
We call this problem \emph{verifiably safe off-model learning}.
In this paper we introduce a first approach toward obtaining formal safety proofs for off-model learning. 
Our approach consists of two components: (1) a model synthesis phase that constructs a set of candidate models together with provably correct control software,
and (2) a runtime model identification process that selects between available models at runtime in a way that preserves the safety guarantees of all
candidate models.

Model update learning is initialized with a set of models.
These models consist of a set of differential equations that model the environment, 
a control program for selecting actuator inputs, 
a safety property,
and a formal proof that the control program constrains the overall system dynamics in a way that correctly ensures the safety property is never violated.

Instead of requiring the existence of a single accurate initial model,
we introduce
\emph{model updates} as syntactic modifications of the differential equations and control logic of the model.
We call a model update \emph{verification-preserving} if there is a corresponding modification to the formal proof establishing
that the modified control program continues to constrain the system of differential equations in a way that preserves the original model's safety properties.

Verification-preserving model updates are inspired by the fact that different parts of a model serve different roles.
The continuous portion of a model is often an assumption about how the world behaves,
and the discrete portion of a model is derived from these equations and the safety property.
For this reason, many of our updates inductively synthesize ODEs (i.e., in response to data from previous executions of the system)
and then deductively synthesize control logic from the resulting ODEs and the safety objective.

Our contributions enabling verifiably safe off-model learning include:
\begin{enumerate*}
\item A set of verification preserving model updates (VPMUs) that 
systematically update differential equations, control software, and safety proofs in a way that preserves verification guarantees
while taking into account possible deviations between an initial model and future system behavior.
\item A reinforcement learning algorithm, called model update learning ($\mu$learning), that explains how to transfer safety proofs for a set of feasible models to a learned policy. The learned policy will actively attempt to falsify models at runtime in order to reduce the safety constraints on actions.
\end{enumerate*}
These contributions are evaluated on a set of hybrid systems control tasks. Our approach uses a combination of program repair, system identification, offline theorem proving, and model monitors to obtain formal safety guarantees for systems in which a single accurate model is not known at design time. This paper fully develops an approach based on a general idea that was first presented in an invited vision paper on Safe AI for CPS by the authors  \cite{itc18}.

The approach described in this paper is model-based but does not assume that a single correct model is known at design time. Model update learning allows for the possibility that all we can know at design time is that there are many feasible models, one of which might be accurate. Verification-preserving model updates then explain how a combination of data and theorem proving can be used at design time to enrich the set of feasible models.

We believe there is a rich space of approaches toward safe learning in-between model-free reinforcement learning (where formal safety guarantees are unavailable) and traditional model-based learning that assumes the existence of a single ideal model. This paper provides a first example of such an approach by leveraging inductive data and deductive proving at both design time and runtime.

The remainder of this paper is organized as follows. We first review the logical foundations underpinning our approach. We then introduce verification-preserving model updates and discuss how experimental data may be used to construct a set of explanatory models for the data. After discussing several model updates, we introduce the $\mu$learning algorithm that selects between models at runtime. Finally, we discuss case studies that validate both aspects of our approach. We close with a discussion of related work.

\section{Background}
\label{sec:background}

This section reviews existing approaches toward safe on-model learning and discusses the fitness of each approach
for obtaining guarantees about off-model learning and then introduces the specification language and logic used throughout the rest of this paper.

Alshiekh et al. and Hasanbeig et al. propose approaches based on Linear Temporal Logic \cite{DBLP:conf/aaai/AlshiekhBEKNT18,HasanbeigKroening}.
Alshiekh et al. synthesize monitoring conditions based upon a safety specification and an environmental abstraction. In terms of this formalism, the goal of off-model learning is to systematically expand the environmental abstraction based upon both design-time insights about how the system's behavior might change over time and also based upon observed data at runtime. 
Jansen et al. extend the approach of Alshiekh et al. by observing that constraints should adapt whenever runtime data suggests that a safety constraint is too restrictive to allow progress toward an over-arching objective \cite{DBLP:journals/corr/bettina20182}.
Herbert et al. address the related problem of safe motion planning by using offline reachability analysis of pursuit-evasion games to pre-compute 
an overapproximate monitoring condition that is used to constrain an online planner \cite{DBLP:conf/icra/FridovichKeilH18,DBLP:conf/cdc/HerbertCHBFT17}.

The above-mentioned approaches have an implicit or explicit environmental model.
Even when these environmental models are accurate, reinforcement learning is still necessary
because these models focus exclusively on safety and are often nondeterministic. Resolving this nondeterminism
in a way that is not only safe but is also effective at achieving other high-level objectives is a task that is well-suited to reinforcement learning.

We are interested in how to provide formal safety guarantees even when there is not a single accurate model available at design time.
Achieving this goal requires two novel contributions. 
We must first
find a way to generate a robust set of feasible models given some combination of an initial model and data on previous runs of the system
(because formal safety guarantees are stated with respect to a model).
Given such a set of feasible models, we must then learn how to safely identify which model is most accurate so that the system is not over-constrained at runtime.

To achieve these goals, we build on the safe learning work for a single model by Fulton et al. \cite{aaai18}.
We choose this approach as a basis for verifiably safe learning because we are interested in safety-critical systems that combine discrete and continuous dynamics,
because we would like to produce explainable models of system dynamics (e.g., systems of differential equations as opposed to large state machines), 
and, most importantly, because our approach requires the ability to systematically modify a model together with that model's safety proof.

Following \cite{aaai18}, we recall Differential Dynamic Logic \cite{DBLP:journals/jar/Platzer08,DBLP:conf/lics/Platzer12a}, a logic for verifying properties about safety-critical hybrid systems control software,
the ModelPlex synthesis algorithm in this logic \cite{DBLP:journals/fmsd/MitschP16}, and the \KeYmaeraX theorem prover \cite{DBLP:conf/cade/FultonMQVP15} that will allow us to systematically modify models and proofs together.

Hybrid (dynamical) systems  \cite{DBLP:conf/hybrid/AlurCHH92,DBLP:conf/lics/Platzer12a} are mathematical models 
that incorporate both discrete and continuous dynamics.
Hybrid systems are excellent models for safety-critical control tasks that combine the discrete dynamics of control software with the continuous
motion of a physical system such as an aircraft, train, or automobile.
Hybrid programs \cite{DBLP:journals/jar/Platzer08,DBLP:conf/lics/Platzer12a,DBLP:journals/jar/Platzer17} are a programming language for hybrid systems.
The syntax and informal semantics of hybrid programs is summarized in Table~\ref{tab:hps}.
The continuous evolution program is a continuous evolution along the differential equation system $x_i'=\theta_i$ for an arbitrary duration within the region described by formula $F$.

\begin{table*}
  \caption{Hybrid Programs} \label{tab:hps}
  \centering
    \begin{tabular}{l | l}
    Program Statement & Meaning \\
    \hline
    $\alpha;\beta$ & Sequentially composes $\beta$ after $\alpha$. \\
    $\alpha \cup \beta$ & Executes either $\alpha$ or $\beta$ nondeterministically.\\
    $\alpha^*$ & Repeats $\alpha$ zero or more times nondeterministically.\\
    $x := \theta$ & Evaluates term $\theta$ and assigns result to variable $x$. \\
    $x := *$ & Nondeterministically assign arbitrary real value to $x$. \\
    $\{x_1'=\theta_1,...,x_n'=\theta_n \& F\}$ & Continuous evolution for any duration within domain $F$. \\
    $?F$ & Aborts if formula $F$ is not true.
  \end{tabular}
\end{table*}

\paragraph{Hybrid Program Semantics}
The semantics of the hybrid programs described by \rref{tab:hps} are given in terms of transitions between states \cite{DBLP:conf/lics/Platzer12a,DBLP:journals/jar/Platzer17}, where a state $s$ assigns a real number $s(x)$ to each variable $x$. 
We use $s\den{t}$ to refer to the value of a term $t$ in a state $s$.
The semantics of a program $\alpha$, written $\den{\alpha}$, is the set of pairs $(s_1, s_2)$ for which state $s_2$ is reachable by running $\alpha$ from state $s_1$.
For example, 
$\den{x:=t_1 \cup x:=t_2}$ is:
$$\{(s_1, s_2) \,|\, s_1{=}s_2 \text{ except } s_2(x)  {=} s_1\den{t_1}\}
~\cup~
\{(s_1, s_2) \,|\, s_1{=}s_2 \text{ except } s_2(x)  {=} s_1\den{t_2}\}$$
for a hybrid program $\alpha$ and state $s$ where $\den{\alpha}(s)$ is set of all states $t$ such that $(s,t) \in \den{\alpha}$.

\paragraph{Differential Dynamic Logic}
Differential dynamic logic ($\dL$) \cite{DBLP:journals/jar/Platzer08,DBLP:conf/lics/Platzer12a,DBLP:journals/jar/Platzer17} is the dynamic logic of hybrid programs.
The logic associates with each hybrid program $\alpha$ modal operators $\dibox{\alpha}$ and $\didia{\alpha}$, which express state reachability properties of $\alpha$.
The formula $\dibox{\alpha} \phi$ states that the formula $\phi$ is true in \emph{all} states reachable by the hybrid program $\alpha$, and the formula $\didia{\alpha}\phi$ expresses that the formula $\phi$ is true after \emph{some} execution of $\alpha$.
The \dL formulas are generated by the grammar
\begin{align*}
	\phi ~::=~ &\theta_1 \backsim \theta_2\ |\ \lnot \phi\ |\ \phi \land \psi\ |\ \phi \lor \psi\ |\ \phi \limply \psi\  |\ \forall x\, \phi\ |\ \exists x\, \phi
  |\ \dibox{\alpha}\phi\ |\ \didia{\alpha}\phi
\end{align*}
where $\theta_i$ are arithmetic expressions over the reals, $\phi$ and $\psi$ are formulas, $\alpha$ ranges over hybrid programs, and $\backsim$ is a comparison operator $=,\neq,\geq,>,\leq,<$.
The quantifiers quantify over the reals.
We denote by $s \models \phi$ the fact that formula $\phi$ is true in state $s$; e.g.,
we denote by $s \models \dbox{\alpha}{\phi}$ the fact that $(s,t) \in \den{\alpha}$ implies $t \models \phi$ for all states $t$.
Similarly, $\infers \phi$ denotes the fact that $\phi$ has a proof in \dL.
When $\phi$ is true in every state (i.e., valid) we simply write~ $\models \phi$.

\begin{example}[Safety specification for straight-line car model]\label{ex:linearCarSpec}
\begin{equation*}
  \underbrace{v {\ge} 0 \land A {>} 0}_{\text{initial condition}}  \limply
  \dbox{
    \prepeat{
      \big(\underbrace{(a{:=}A \cup a{:=}0)}_{\textit{ctrl}}\ ;
      \underbrace{\{p'{=}v,\, v'{=}a\}}_{\textit{plant}}\big)
    }
  }
  {\hspace{-6pt}\underbrace{v {\ge} 0}_{\text{post cond.}}}
\end{equation*}
\end{example}

This formula states that if a car begins with a non-negative velocity, then it will also always have a non-negative velocity after repeatedly choosing new acceleration ($A$ or $0$), or coasting and moving for a nondeterministic period of time.

Throughout this paper, we will refer to sets of actions. An \textbf{action} is simply the effect of a loop-free deterministic discrete program without tests. For example, the programs $a{:=}A$ and $a {:=} 0$ are the actions available in the above program. Notice that \textbf{actions} can be equivalently thought of as mappings from variables to terms. We use the term action to refer to both the mappings themselves and the hybrid programs whose semantics correspond to these mappings. For an action $u$, we write $u(s)$ to mean the effect of taking action $u$ in state $s$; i.e., the state $t$ such that $(s,t) \in \den{u}$.

\paragraph{ModelPlex}
Safe off-model learning requires noticing, at runtime, when a system  deviates from model assumptions. Therefore, our approach depends upon the ability to check, at runtime, whether the current state of the system can be explained by a hybrid program.

The \KeYmaeraX theorem prover implements the ModelPlex algorithm \cite{DBLP:journals/fmsd/MitschP16}. ModelPlex constructs an implication or equivalence proof between a \dL formula of the form 
$\phi \limply \dbox{\alpha^*}{\psi}$
and a monitoring condition expressed as a formula of quantifier-free real arithmetic. The monitoring condition is then used to extract provably correct monitors that check whether observed transitions comport with modeling assumptions. ModelPlex can produce monitors that enforce models of control programs as well as monitors that check whether the model's ODEs comport with observed state transitions.

ModelPlex \emph{controller monitors} are boolean functions that return false if the controller portion of a hybrid systems model has been violated. A \emph{controller monitor} for a model $\{\tvar{ctrl}; \tvar{plant}\}^*$ is a function $\tvar{cm} : \mathcal{S} \times \mathcal{A} \rightarrow \mathbb{B}$
from states $\mathcal{S}$ and actions $\mathcal{A}$ to booleans $\mathbb{B}$ such that 
if $\tvar{cm}(s,a)$ then $(s, a(s)) \in \den{ctrl}$. We sometimes also abuse notation 
by using controller monitors as an implicit filter on $\mathcal{A}$; i.e., $\tvar{cm} : \mathcal{S} \rightarrow \mathcal{A}$ such that
 $a \in \tvar{cm}(s)$ iff $\tvar{cm}(s,a)$ is true.

ModelPlex also produces \emph{model monitors}, which check whether the model is accurate.
A \emph{model monitor} for a safety specification $\phi \limply \dibox{\alpha^*}\psi$ is a function
$\tvar{mm} : \mathcal{S} \times \mathcal{S} \rightarrow \mathbb{B}$ such that 
$(s_0, s) \in \den{\alpha}$ if
$\tvar{mm}(s_0, s)$.
For the sake of brevity, we also define $\tvar{mm} : \mathcal{S} \times \mathcal{A} \times \mathcal{S} \rightarrow \mathbb{B}$ 
as the model monitor applied after taking an action ($a \in A$) in a state and then following the plant in a model of form $\alpha \equiv \tvar{ctrl};\tvar{plant}$.
Notice that if the model has this canonical form and if 
if $\tvar{mm}(s, a, a(s))$ for an action $a$, then $\tvar{cm}(s, a(s))$.

The KeYmaera~X system is a theorem prover \cite{DBLP:conf/cade/FultonMQVP15} that provides a language called Bellerophon for scripting proofs of \dL formulas
\cite{bellerophon}. Bellerophon programs, called tactics, construct proofs of \dL formulas.
This paper proposes an approach toward updating models in a way that preserves safety proofs.
Our proposed approach simultaneously changes a system of differential equations, control software expressed as a discrete loop-free program,
and the formal proof that the controller properly selects actuator values such that desired safety constraints are preserved throughout the flow of a 
system of differential equations.

\section{Verification-Preserving Model Updates}
\label{sec:vpmus}

A \emph{verification-preserving model update} (VPMU) is a transformation of a hybrid program  
accompanied by a proof that the transformation preserves key safety properties \cite{itc18}.
VPMUs capture situations in which a model and/or a set of data can be updated in a way that captures
possible runtime behaviors which are not captured by an existing model.

\begin{definition}[VPMU]
\label{def:vpmu}
A \emph{verification-preserving model update} is a mapping
which takes as input an initial \dL formula $\varphi$ with an associated Bellerophon tactic $\tvar{e}$ of $\varphi$,
and produces as output a new \dL formula $\psi$ and a new Bellerophon tactic $\tvar{f}$ 
such that $\tvar{f}$ is a proof of $\psi$.
\end{definition}

Before discussing our VPMU library, we consider how a set of feasible models computed using VPMUs can be 
used to provide verified safety guarantees for a family of reinforcement learning algorithms. The primary challenge is to maintain safety with respect
to all feasible models while also avoiding overly conservative monitoring constraints by falsifying some of these models at runtime.

\section{Verifiably Safe RL with Multiple Models}
\label{sec:mulearning}

VPMUs may be applied any time at which system designers can characterize likely ways in which an existing model will deviate from reality. Model updates are easiest to apply at design time because of the computational overhead of computing both model updates and corresponding proof updates. This section introduces model update learning, which explains how to take a set of models generated using VPMUs at design time to provide safety guarantees at runtime.

Model update learning is based on a simple idea: begin with a set of \emph{feasible models}
and act safely with respect to all feasible models. Whenever a model does not comport with observed dynamics,
the model becomes infeasible and is therefore removed from the set of feasible models.
We introduce two variations of $\mu$learning: 
a basic algorithm that chooses actions without considering the underlying action space,
and an algorithm that prioritizes actions that rule out feasible models (adding an \emph{eliminate} choice to the classical explore/exploit tradeoff \cite{sutton.barto:reinforcement}).

All $\mu$learning algorithms use monitored models; i.e., models equipped with ModelPlex controller monitors and model monitors.

\begin{definition}[Monitored Model]
A \textbf{monitored model} is a tuple $(m,cm,mm)$ where
$m$ is a \dL formula of the form $\tvar{init} \limply \dibox{\{\tvar{ctrl}; \tvar{plant}\}^*}\tvar{safe}$,
and
$\tvar{ctrl}$ is a loop-free program, and the entire formula $m$ contains exactly one modality,
and the formulas $cm$ and $mm$ are the control monitor and model monitor corresponding to $m$, as defined in \rref{sec:background}.
\end{definition}

Monitored models may have a continuous action space because of both tests and the nondeterministic assignment operator. We sometimes introduce additional assumptions on the structure of the monitored models. A monitored model over a finite action space is a monitored model where $\{(s,t) : (s, t) \in \den{\tvar{ctrl}}\}$ is finite for all $s \in S$. A time-aware monitored model is a monitored model such that the differential equations contain a local clock which is reset at each control step.

Model update learning, or $\mu$learning, leverages verification-preserving model updates to 
maintain safety while selecting an appropriate environmental model. We now state and prove key
safety properties about the $\mu$learning algorithm.

\begin{definition}[$\mu$learning Process]
\label{def:mulearning}
A learning process $P_M$ for a finite set of monitored models $M$ is defined as a tuple of countable sequences
$(\textbf{U}, \textbf{S}, \textbf{Mon})$ where 
$\textbf{U}$ are actions in a finite set of actions $\mathcal{A}$ (i.e., mappings from variables to values),
elements of the sequence $\textbf{S}$ are states, and
$\textbf{Mon}$ are monitored models with $\textbf{Mon}_0 = M$.
Let
$specOK_m(\textbf{U},\textbf{S},i) \equiv \tvar{cm}(\textbf{S}_i, \textbf{U}_i) \lor \lnot \tvar{mm}(\textbf{S}_{i-1}, \textbf{U}_{i-1}, \textbf{S}_i) $
where $\tvar{cm}$ and $\tvar{mm}$ are the monitors corresponding to the model $m$. Let $specOK$ always return true for $i=0$.

A \textbf{$\mu$learning process} is a learning process satisfying the following additional conditions:
\begin{enumerate*}
\item action availability: in each state $\textbf{S}_i$ there is at least one action $u$ such that for all $m \in \textbf{Mon}_i$, $u \in \tvar{specOK}_m(\textbf{U},\textbf{S},i)$,
\item actions are safe for all feasible models: 
$\textbf{U}_{i+1} \in \{u \in A \,|\, \forall (m,\tvar{cm},\tvar{mm}) \in \textbf{Mon}_i , \tvar{cm}(\textbf{S}_{i}, u) \}$,
\item feasible models remain in the feasible set: 
if $(\varphi, \tvar{cm}, \tvar{mm}) \in \textbf{M}_i$ 
and $\tvar{mm}(\textbf{S}_{i}, \textbf{U}_{i}, \textbf{S}_{i+1})$ 
then $(\varphi, \tvar{cm}, \tvar{mm}) \in \textbf{Mon}_{i+1}$.
\end{enumerate*}
\end{definition}

Note that $\mu$learning processes are defined over an environment $E : \mathcal{A} \times \mathcal{S} \rightarrow \mathcal{S}$
that determines the sequences $\textbf{U}$ and $\textbf{S}$\footnote{Throughout the paper, we denote by \textbf{S} a specific sequence of states and by $\mathcal{S}$ the set of all states.},
so that $\textbf{S}_{i+1} = E(\textbf{U}_i, \textbf{S}_i)$.
In our algorithms, the set $\textbf{Mon}_{i}$ never retains elements that are inconsistent with the observed
dynamics at the previous state. We refer to the set of models in $\textbf{Mon}_i$ as the set of feasible models for $s_i$.

Notice that the safe actions constraint is not effectively checkable without extra assumptions on the range of parameters.
Two canonical choices are discretizing options for parameters or including an effective identification process
for parameterized models.

Our safety theorem focuses on time-aware $\mu$learning processes, i.e., those whose models are all time-aware; similarly, a \emph{finite action space $\mu$learning process} is a $\mu$learning process in which all models $m\in M$ have a finite action space. The basic correctness property for a $\mu$learning process is the safe reinforcement learning condition: the system never takes unsafe actions.

\begin{definition}[$\mu$learning process with an accurate model]\label{def:accurateModel}
Let $P_M = (\textbf{S}, \textbf{U}, \textbf{Mon})$ be a $\mu$learning process.
Assume there is some element $m^* \in \textbf{Mon}_0$ with the following properties.
First,
$m^* \equiv (\tvar{init}_m \limply \dibox{\{ \tvar{ctrl}_m; \tvar{plant}_m  \}^*}\tvar{safe})$.
Second, $\infers m^*$.
Third, $(s, u(s)) \in \den{\tvar{ctrl}_m}$ implies $(u(s), E(u,s)) \in \den{plant}$
for a mapping $E : \mathcal{S} \times \mathcal{A} \rightarrow \mathcal{S}S$ from states and actions to new states.
When only one element of $\textbf{Mon}_0$ satisfies these properties we call that element $m^*$ the
\emph{distinguished} and/or \emph{accurate} model and say that the process $P_M$ is \emph{accurately modeled with respect to a mapping $E$}.
\end{definition}

The mapping $E$ is often called an environment,
and we will often elide the environment for which the process $P_M$ is accurate when it is obvious from context. 

\begin{theorem}[Safety]
If $P_M$ is a $\mu$learning process with an accurate model,
then $\textbf{S}_i \models \tvar{safe}$ for all $0 < i < |\textbf{S}|$.
\end{theorem}
\begin{proof}
Let $m^*$ be the distinguished model for $P_M = (\textbf{S}, \textbf{U}, \textbf{Mon})$.
Proceed by induction on the length of $\textbf{S}$ with the hypothesis that
$\textbf{S}_i \models \tvar{safe}$ and $m^* \in \textbf{Mon}_i$.
Let $E$ be the environment with respect to which $P_M$ is defined.

By the definition of a $\mu$learning process with an accurate model,
$\textbf{S}_0 \models \tvar{safe}$ and
$m^* \in \textbf{Mon}_0$.

Now, assume $\textbf{S}_i \models \tvar{safe}$ and $m^* \in \textbf{Mon}_i$.
By the definition of a $\mu$learning process, we therefore know that either
\[
\textit{cm}^*(\textbf{S}_{i-1}, \textbf{U}_{i-1}(\textbf{S}_{i-1}))
\]
or else
\[
\lnot \textit{mm}^*(\textbf{S}_{i-1}, \textbf{U}_{i-1}(\textbf{S}_{i-1}), \textbf{S}_{i})
\]
However, the second cannot be the case due to the hypothesis that $m^*$ is the accurate model (\rref{def:accurateModel}).
Therefore, it must be that
$(\textbf{S}_i, \textbf{U}_i(\textbf{S}_i)) \in \den{\tvar{ctrl}_{m^*}}$.
From this fact, \rref{def:accurateModel}, and the fact that $m^*$ is distinguished, it follows that
$(\textbf{U}_i(s), E(\textbf{U}_i,\textbf{S}_i)) \in \den{plant}$.
This, together with the fact that $\vdash_{\dL} \varphi^*$, the shape assumption on $\varphi^*$, and the soundness of \dL,
implies that $\textbf{S}_{i+1} \models \tvar{safe}$.

What remains to be shown is that $m^* \in \textbf{Mon}_{i+1}$.
Notice that 
$$\textit{mm}^*(\textbf{S}_{i}, \textbf{U}_{i}(\textbf{S}_{i}), \textbf{S}_{i+1})$$
must be true because $m^*$ is the accurate model.
By the inductive hypothesis know also that $m^* \in \textbf{Mon}_{i}$.
By definition, \emph{feasible models remain in the feasible set} (\rref{def:mulearning});
i.e., the above two facts establish that $m^* \in \textbf{Mon}_{i+1}$.
\end{proof}


\rref{listing:mulearning} defines a $\mu$learning algorithm that produces a $\mu$learning process.
The inputs are: \begin{enumerate*}[label=\textbf{(\itshape\alph*\upshape)}]
\item A set of $\tvar{M}$ models each with a method $\tvar{m.models} : \mathcal{S} \times \mathcal{A} \times \mathcal{S} \rightarrow \mathbb{B}$ which implements the evaluation of its model monitor in the given previous and next state and actions and a method $\tvar{m.safe} : \mathcal{S} \times \mathcal{A} \rightarrow \mathbb{B}$ which implements evaluation of its controller monitor,
\item an action space $\tvar{A}$ and an initial state $\tvar{init}\in S$,
\item an environment function $\tvar{env} : \mathcal{S} \times \mathcal{A} \rightarrow \mathcal{S} \times \mathbb{R}$ that computes state updates and rewards in response to actions, and
\item a function $\tvar{choose} : \wp(\mathcal{A})\to \mathcal{A}$ that selects an action from a set of available actions and $\tvar{update}$ updates a table or approximation. Our approach is generic and works for any reinforcement learning algorithm; therefore, we leave these functions abstract.
\end{enumerate*}
It augments an existing reinforcement learning algorithm, defined by $\tvar{update}$ and $\tvar{choose}$, by restricting the action space at each step so that actions are only taken if they are safe with respect to \emph{all} feasible models. The feasible model set is updated at each control set by removing models that are in conflict with observed data.

The $\mu$learning algorithm rules out incorrect models from the set of possible models by taking actions and observing the results of those actions. Through these experiments, the set of relevant models is winnowed down to either the distinguished correct model $m^*$, or a set of models $M^*$ containing $m^*$ and other models that cannot be distinguished from $m^*$.

\noindent%
\begin{minipage}{\linewidth}
\begin{lstlisting}[language=Python,caption={\label{listing:mulearning}The basic $\mu$learning algorithm}, mathescape]
def $\mu$learn(M,A,init,env,choose,update):
  s_pre = s_curr = init
  act   = None
  while(not done(s_curr)):
    if act is not None:
      M = {m $\in$ M : m.models(s_pre,act,s_curr)}
    avail = {a $\in$ A : $\forall$ m $\in$ M, m.safe(s_curr, a)}
    act = choose(avail)
    s_pre = s_curr
    (s_curr, reward) = env(s_curr, act)
    update(s_pre, act, s_curr, reward)
\end{lstlisting}
\end{minipage}

\subsection{Active Verified Model Update Learning}

Removing models from the set of possible models relaxes the monitoring condition, allowing less conservative and more accurate control decisions. Therefore, this section introduces an active learning refinement of the $\mu$learning algorithm that prioritizes taking actions that help rule out models $m \in M$ that are not $m^*$. Instead of choosing a random safe action, $\mu$learning prioritizes actions that differentiate between available models. We begin by explaining what it means for an algorithm to perform good experiments.

\begin{definition}[Active Experimentation]
\label{def:goodExperiment}
A $\mu$learning process with an accurate model $m^*$ has \emph{locally active experimentation} provided that: 
if $M_{i}>1$ and there exists an action $a$ that is safe for all feasible models (see \rref{def:mulearning}) in state $s_i$ such that taking action $a$ results in the removal of $m$ from the model set, then $|M_{i+1}| < |M_{i}|$.
Experimentation is $\tvar{er}$-active if the following conditions hold:
there exists an action $a$ that is safe for all feasible models (see \rref{def:mulearning}) in state $s_i$, and taking action $a$ results in the removal of $m$ from the model set, then $|M_{i+1}| < |M_{i}|$ with probability $0 < \tvar{er} < 1$.
\end{definition}

\begin{definition}[Distinguishing Actions]
Consider a $\mu$learning process $(\textbf{U},\textbf{S},\textbf{Mon})$ with an accurate model $m^*$ (see \rref{def:accurateModel}). An action $a$ distinguishes $m$ from $m^*$ if 
$a = \textbf{U}_i$, $m \in \textbf{Mon}_i$ and $m \not \in \textbf{Mon}_{i+1}$ for some $i>0$.
\end{definition}

The \emph{active $\mu$learning algorithm} uses model monitors to select distinguishing
actions, thereby performing active experiments which winnow down the set of feasible models.
The inputs to \texttt{active-$\mu$learn} are the same as those to \rref{listing:mulearning} with two additions:
\begin{enumerate*}
\item models are augmented with an additional
prediction method $\tvar{p}$ that returns the model's prediction of the next state given the current state, a candidate action, and a time duration.
\item An elimination rate $\tvar{er}$ is introduced, which plays a similar role as the 
classical explore-exploit rate except that we are now choosing
whether to insist on choosing a good experiment.
\end{enumerate*}
The \tvar{active-$\mu$learn} algorithm is guaranteed to make some progress toward winnowing down the feasible model set whenever $0 < \tvar{er} < 1$.

\begin{theorem}
\label{thm:mupreserve}
Let $P_M = (\textbf{S}, \textbf{U}, \textbf{Mon})$ be a finite action space $\mu$learning process with an accurate model $m^*$.
Then $m^* \in \textbf{Mon}_i$ for all $0 \le i \le |\textbf{Mon}|$.
\end{theorem}
\begin{proof}
Let $m^* = (\varphi^*, \textit{cm}^*, \textit{mm}^*)$ and let $E$ be the environment over which $P_M$ is defined.
By \rref{def:accurateModel}, \[ m^* \in \textbf{Mon}_0 \]

Suppose now that $m^* \in \textbf{Mon}_i$ for some $i \ge 0$.
The crucial observation is that $\textit{mm}^*(\textbf{S}_i, \textbf{U}_i, \textbf{S}_{i+1})$,
which will directly imply that $m^* \in \textbf{Mon}_{i+1}$ due to the fact that feasible models remain in the feasible set (\rref{def:mulearning}).

So, it suffices to show that $\textit{mm}^*(\textbf{S}_i, \textbf{U}_i, \textbf{S}_{i+1})$.
However, notice that this follows directly from the definition of an accurate model (\rref{def:accurateModel}).
\end{proof}

\begin{theorem}
\label{thm:removeWrongModels}
Let $P_M$ be a finite action space $\tvar{er}$-active $\mu$learning process
under environment $E$ and with an accurate model $m^*$.
Assume $P_M$ has an accurate model $m^*$
and that every other model $m \not = m^* \in \textbf{M}_0$ has in each state $s$ an action $a_s$ that is safe for all models and distinguishes $m$ from $m^*$.\\
Then $\lim_{i \rightarrow \infty} \text{Pr}(m \not \in M_i) = 1$.
\end{theorem}
\begin{proof}
Consider $P_M = (\textbf{S}, \textbf{U}, \textbf{Mon})$.
Note that $m^* \in \textbf{Mon}_i$ for all $0 \le i \le |\textbf{Mon}|$ is directly implied by \rref{thm:mupreserve}.
Let $k = |M_i| - 1$ be the number of non-$m^*$ elements in $\textbf{Mon}_i$.
It suffices to show that $\lim_{i \rightarrow \infty} \text{Pr}(k = 0) = 1$.
Notice that because each $m \not = m^*$ is falsifiable at step $i$, 
the probability that $k-1$ non-$m^*$ elements remain in after $n$ steps is $\sum_{j=0}^{k-1} \tvar{er}^j (1-\tvar{er})^{n-j}$
which tends to $0$ as $n \rightarrow \infty$.
\end{proof}

\begin{cor}
Let $P_M = (\textbf{S}, \textbf{U}, \textbf{Mon})$ be a finite action space $\tvar{er}$-active $\mu$learning process
under environment $E$ and with an accurate model $m^*$.
If \emph{each} model $m \not = m^*$ has in each state $s$ an action $a_s$ that is safe for all models and distinguishes $m$ from $m^*$,
then $\textbf{Mon}$ converges to $\{m^*\}$ a.s.
\end{cor}
\begin{proof}
The conclusion follows from \rref{thm:removeWrongModels} as long as $m^*$ is never removed.
Proceed by induction on the length of $\textbf{Mon}$.
In the base case, recall that  $m^* \in \textbf{Mon}_0$.
Suppose, then, that $m^* \in \textbf{Mon}_i$.
We assume $m^* = (\varphi^*, \textit{cm}^*, \textit{mm}^*)$ is accurately modeled with respect to $E$,
which by \rref{def:accurateModel} implies $\textit{mm}^*(\textbf{S}_{i}, \textbf{U}_i, \textbf{S}_{i+1})$ for all elements in the sequences $\textbf{U}$ and $\textbf{S}$.
Therefore, $\textbf{Mon}_{i+1}$ must contain $m^*$ by the inductive hypothesis applied to the third condition (``feasible models remain in the feasible set") of \rref{def:mulearning}.
\end{proof}

Although locally active experimentation is not
strong enough to ensure that $P_M$ eventually converges
to a minimal set of models\footnote{%
$x\ge 0 \land t=0 \limply \dbox{\{\{?t=0;x:=1 \cup x:=0\}; \{x'=F,t'=1\}\}^*}{\,x\ge0}$ with the parameters $F=0,F=5, \text{ and } F=x$ are a counter example \cite[Section 8.4.4]{FultonThesis}.}, 
our experimental validation demonstrates that this heuristic is none-the-less effective on some
representative examples of model update learning problems.

\section{A Model Update Library}

So far, we have established how to obtain safety guarantees for reinforcement learning algorithms given a set of formally verified \dL models.
We now turn our attention to the problem of generating such a set of models by systematically modifying \dL formulas and their corresponding Bellerophon tactical proof scripts. This section introduces five generic model updates that provide a representative sample of the kinds of computations that can be performed on models and proofs
to predict and account for runtime model deviations\footnote{Extended discussion of these model updates is available in \cite[Chapters 8 and 9]{FultonThesis}. Implementations are available at \url{https://nfulton.org/vpmu}.}.

The simplest example of a VPMU instantiates a parameter whose value is not known at design time
but can be determined at runtime via system identification. 
Consider a program $p$ modeling a car whose acceleration depends upon both a known control input $accel$ and parametric values for maximum
braking force $-B$ and maximum acceleration $A$.
Its  proof is
$$\tvar{implyR}(1); \tvar{loop}(\tvar{pos} - \tvar{obsPos} > \frac{\tvar{vel}^2}{2B}, 1); \tvar{onAll}(\tvar{master})$$
This model and proof can be updated with concrete experimentally determined values for each parameter by uniformly substituting the variables $B$ and $A$
with concrete values in both the model and the tactic.

The \textbf{Automatic Parameter Instantiation} update improves the basic parameter instantiation update by automatically detecting which variables are parameters and then constraining instantiation of parameters by identifying relevant initial conditions.

The \textbf{Replace Worst-Case Bounds with Approximations}
update improves models designed for the purpose of safety verification.
Often a variable occurring in the system is bounded above (or below) by its worst-case value.
Worst-case analyses are sufficient for establishing safety but are often overly conservative.
The approximation model update replaces worst-case bounds with approximate bounds
obtained via series expansions. The proof update then introduces a tactic on each branch of the proof that establishes our approximations are
upper/lower bounds by performing.

Models often assume perfect sensing and actuation. A common way of robustifying a model is to add a piecewise constant noise term
to the system's dynamics. Doing so while maintaining safety invaraints requires also updating the controller so that safety envelope 
computations incorporate this noise term.
The \textbf{Add Disturbance Term} update introduces noise terms to differential equations, 
systematically updates controller guards, and modifies the proof accordingly.

Uncertainty in object classification is easy to model in terms of sets of feasible models. In the simplest case, a robot might need to avoid an obstacle that is either static, moves in a straight line, or moves sinusoidally. Our generic model update library contains an update that changes the model by making a static point $(x,y)$ dynamic. For example, one such update introduces the equations $\{x'=-y, y'=-x\}$ to a system of differential equations in which the variables $x,y$ do not have differential equations. The controller is updated so that any statements about separation between $(a,b)$ and $(x,y)$ require global separation of $(a,b)$ from the circle on which $(x,y)$ moves. The proof is also updated by prepending to the first occurrence of a differential tactic on each branch with a sequence of differential cuts that characterize circular motion.

Model updates also provide a framework for characterizing algorithms that combine model identification and controller synthesis. One example is our synthesis algorithm for systems whose ODEs have solutions  in a decidable fragment of real arithmetic (a subset of linear ODEs). Unlike other model updates, we do not assume that any initial model is provided; instead, we learn a model (and associated control policy) entirely from data. The \textbf{Learn Linear Dynamics} update takes as input: (1) data from previous executions of the system, and (2) a desired safety constraint. From these two inputs, the update computes a set of differential equations $\texttt{odes}$ that comport with prior observations, a corresponding controller $\texttt{ctrl}$ that enforces the desired safety constraint with corresponding initial conditions $\texttt{init}$, and a Bellerophon tactic $\texttt{prf}$ which proves $\texttt{init} \limply \dibox{\{\texttt{ctrl}; \texttt{odes}\}^*}\texttt{safe}$. The full mechanism is beyond the scope of this paper and is explained in greater detail by Chapter 9 of \cite{FultonThesis}.

\paragraph{Significance of Selected Updates}
The updates described in this section demonstrate several possible modes of use for VPMUs and $\mu$learning. VPMUS can update existing models to account for systematic modeling errors (e.g., missing actuator noise or changes in the dynamical behavior of obstacles). VPMUs can automatically optimize control logic in a proof-preserving fashion. VPMUS can also be used to generate accurate models and corresponding controllers from experimental data made available at design time, without access to any prior model of the environment.

\section{Experimental Validation}
\label{sec:experimentalValidation}

The $\mu$learning algorithms introduced in this paper are designed to answer the following question: 
given a set of possible models that contains the one true model, how can we \emph{safely} perform a set of experiments 
that allow us to efficiently discover a minimal safety constraint?
In this section we present several experiments which demonstrate the use of $\mu$learning in safety-critical settings.

Overall, these experiments empirically validate our theorems by demonstrating that $\mu$learning processes with accurate models do not violate safety constraints.

Our simulations use a conservative discretization of the hybrid systems models, and we translated monitoring conditions by hand into Python from ModelPlex's C output. Although we evaluate our approach in a research prototype implemented in Python for the sake of convenience, there is a verified compilation pipeline for models implemented in \dL that eliminates uncertainty introduced by discretization and hand-translations \cite{DBLP:conf/pldi/BohrerTMMP18}.

\noindent\textbf{Adaptive Cruise Control}
Adaptive Cruise Control (ACC) is a common feature in new cars. 
ACC systems change the speed of the car in response to the changes in the speed of traffic in front of the car;
e.g., if the car in front of an ACC-enabled car begins slowing down, then the ACC system will decelerate to match
the velocity of the leading car. Our first set of experiments consider a simple linear model of ACC
in which the acceleration set-point is perturbed by an unknown parameter $p$; i.e., the relative position of the
two vehicles is determined by the equations
$\text{pos}_{\text{rel}}' = \text{vel}_{\text{rel}}, \text{vel}_{\text{rel}}' = \text{acc}_{\text{rel}} \label{eqn:relodes}$.

In \cite{aaai18}, the authors consider the collision avoidance problem when a noise term is added so that $\text{vel}_{\text{rel}}' = p\text{acc}_{\text{rel}}$.
We are able to outperform the approach in \cite{aaai18} by combining the \textbf{Add Noise Term} and \textbf{Parameter Instantiation} updates; we outperform
in terms of both avoiding unsafe states and in terms of cumulative reward. These two updates allow us to 
insert a multiplicative noise term $p$ into these equations, synthesize a provably correct controller, and then
choose the correct value for this noise term at runtime. Unlike \cite{aaai18}, $\mu$learning avoids all safety violations.
The graph in \rref{fig:experiments}
compares the Justified Speculative Control approach of \cite{aaai18} to our approach in terms of cumulative reward;
in addition to substantially outperforming the JSC algorithm of \cite{aaai18}, $\mu$learning also avoids 204 more
crashes throughout a 1,000 episode training process.

\vspace{-2em}
\begin{figure}
\begin{center}
\subfloat{%
\centering
\includegraphics[width=0.5\columnwidth]{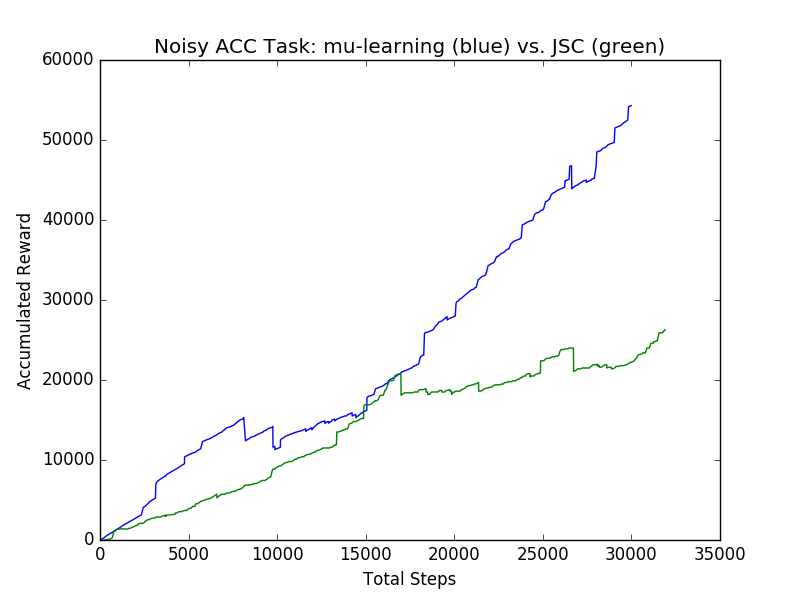}
}
\subfloat{%
\includegraphics[width=0.5\columnwidth]{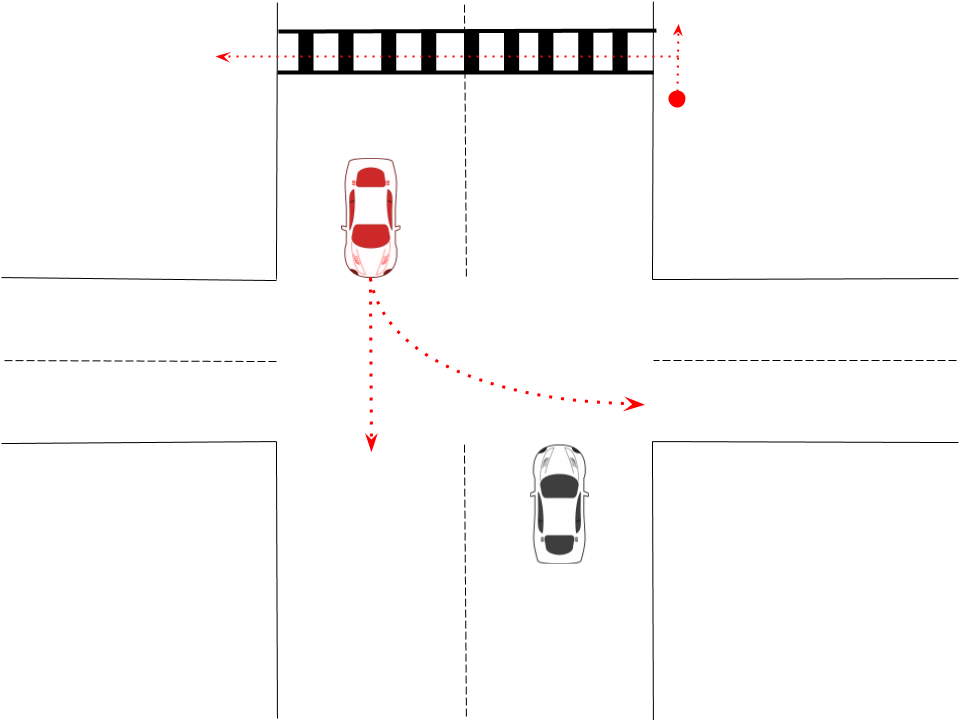}
}
\caption{
\small Left: The cumulative reward obtained by Justified Speculative Control \cite{aaai18} (green) and
$\mu$learning (blue) during training over 1,000 episodes with each episode truncated at 100 steps.
Each episode used a randomly selected error term that remains constant throughout each episode but may change between episodes.
Right: a visualization of the hierarchical safety environment.
}
\label{fig:experiments}
\end{center}
\end{figure}
\vspace{-2em}

\noindent\textbf{A Hierarchical Problem}
Model update learning can be extended
to provide formal guarantees for hierarchical reinforcement learning algorithms \cite{DBLP:journals/deds/BartoM03}.
If each feasible model $m$ corresponds to a subtask,
and if all states satisfying termination conditions for subtask $m_i$  are also safe initial states for any subtask $m_j$ reachable from $m_i$, 
then $\mu$learning directly 
supports safe hierarchical reinforcement learning by re-initializing $M$ to the initial (maximal) model set whenever reaching a termination condition for the current subtask.

We implemented a variant of $\mu$learning that performs this re-initialization and validated this algorithm in an environment where a car must first
navigate an intersection containing another car and then must avoid a pedestrian in a crosswalk (as illustrated in \rref{fig:experiments}).
In the crosswalk case, the pedestrian at $(ped_x, ped_y)$ may either continue to walk along a sidewalk indefinitely or
may enters the crosswalk at some point between $c_{min} \le ped_y \le c_{max}$ (the boundaries of the crosswalk).
This case study demonstrates that 
safe hierarchical reinforcement learning is simply safe $\mu$learning with safe model re-initialization.

\section{Related Work}

Related work falls into three broad categories: 
safe reinforcement learning, runtime falsification, and program synthesis.

Our approach toward safe reinforcement learning differs from existing approaches
that do not include a formal verification component
(e.g., as surveyed by Garc\'{i}a and Fern\'{a}ndez \cite{DBLP:journals/jmlr/GarciaF15}
and the SMT-based constrained learning approach of Junges et al. \cite{DBLP:conf/tacas/Junges0DTK16}) because
we focused on \emph{verifiably} safe learning; i.e., 
instead of relying on oracles or conjectures, constraints 
are derived in a provably correct way from 
formally verified safety proofs.
The difference between verifiably safe learning and safe learning is significant, 
and is equivalent to the difference between verified and unverified software.
Unlike most existing approaches our safety guarantees apply to both the learning process and
the final learned policy.

\rref{sec:background} discusses how our work relates to the few existing approaches toward \emph{verifiably} safe reinforcement learning.
Unlike those \cite{DBLP:conf/aaai/AlshiekhBEKNT18,HasanbeigKroening,DBLP:journals/corr/bettina20182,aaai18},
as well as work on model checking and verification for MDPs \cite{DBLP:conf/qest/HenriquesMZPC12}, we introduce an approach toward verifiably safe off-model learning.
Our approach is the first to combine model synthesis at design time with 
model falsification at runtime so that safety guarantees capture a wide range of possible futures instead of relying on a single accurate environmental model. 
Safe off-model learning is an important problem because autonomous systems must be able to cope with unanticipated scenarios. Ours is the first approach toward verifiably safe off-model learning.

Several recent papers focus on providing safety guarantees for model-free reinforcement learning.
Trust Region Policy Optimization \cite{DBLP:conf/icml/SchulmanLAJM15} defines safety as monotonic policy improvement,
a much weaker notion of safety than the constraints guaranteed by our approach. Constrained Policy Optimization
\cite{DBLP:conf/icml/AchiamHTA17} extends TRPO with guarantees that an agent nearly satisfies safety constraints during learning.  Br\'{a}zdil et al. \cite{DBLP:conf/atva/BrazdilCCFKKPU14} give probabilistic guarantees by performing a heuristic-driven exploration of the model.
Our approach is model-based instead of model-free, and instead of focusing on learning safely without a model we focus on 
identifying accurate models from data obtained both at design time and at runtime. Learning concise dynamical systems representations 
has one substantial advantage over model-free methods:
safety guarantees are stated with respect to an explainable model that captures the safety-critical assumptions about the system's dynamics.
Synthesizing explainable models is important
because safety guarantees are always stated with respect to a model; therefore, engineers must be able to understand inductively synthesized models 
in order to understand what safety properties their systems do (and do not) ensure.

Akazaki et al. propose an approach, based on deep reinforcement learning,
for efficiently discovering defects in models of cyber-physical systems with specifications stated in signal temporal logic \cite{AkazakiCPSfalsification}.
Model falsification is an important component of our approach; however, unlike Akazaki et al., 
we also propose an approach toward obtaining more robust models and explain how runtime falsification can be used to obtain safety guarantees for off-model learning.

Our approach includes a model synthesis phase that is closely related to program synthesis and program repair algorithms \cite{DBLP:conf/aaip/Kitzelmann09,DBLP:journals/tse/GouesNFW12,DBLP:conf/fm/RothenbergG16}.
Relative to work on program synthesis and repair, VPMUs are unique in several ways.
We are the first to explore \emph{hybrid} program repair.
Our approach combines program verification with mutation.
We treat programs as \emph{models} in which one part of the model is 
varied according to interactions with the environment and another part of the model is systematically derived (together with a correctness proof)
from these changes. This separation of the dynamics 
into inductively synthesized models and deductively synthesized controllers
enables our approach toward using programs
as representations of dynamic safety constraints during reinforcement learning.

Although we are the first to explore hybrid program repair, several researchers have explored the problem of
synthesizing hybrid systems from data \cite{AsarinSynthesis,DBLP:conf/hybrid/SadraddiniB18}. This work is closely related to our \textbf{Learn Linear Dynamics} update.
Sadraddini and Belta provide formal guarantees for data-driven model identification and controller synthesis \cite{DBLP:conf/hybrid/SadraddiniB18}.
Relative to this work, our \textbf{Learn Linear Dynamics} update is continuous-time, synthesizes a computer-checked correctness proof 
(as opposed to a least violation criterion implied by the synthesis algorithm, both not independently checked by a theorem prover),
and does not require dynamics to be compact, bounded, and locally connected.
Unlike Asarin et al. \cite{AsarinSynthesis}, our full set of model updates is sometimes capable of synthesizing nonlinear dynamical systems from data (e.g., the static $\rightarrow$ circular update)
and produces computer-checked correctness proofs for permissive controllers.

\section{Conclusions}
This paper introduces an approach toward verifiably safe off-model learning that uses a combination of design-time
verification-preserving model updates and runtime model update learning to provide safety guarantees even when there is no
single accurate model available at design time. We introduced a set of model updates that capture common ways in which models can 
deviate from reality, and introduced an update that is capable of synthesizing ODEs and provably correct controllers without access
to an initial model.
Finally, we proved safety and efficiency theorems for active $\mu$learning and evaluated our approach on  some representative examples of hybrid systems control tasks.
Together, these contributions constitute a first approach toward verifiably safe off-model learning.

\bibliographystyle{splncs04}
\bibliography{vpmu}

\begin{thebibliography}{10}
\providecommand{\url}[1]{\texttt{#1}}
\providecommand{\urlprefix}{URL }
\providecommand{\doi}[1]{https://doi.org/#1}

\bibitem{DBLP:conf/icml/AchiamHTA17}
Achiam, J., Held, D., Tamar, A., Abbeel, P.: Constrained policy optimization.
  In: Precup, D., Teh, Y.W. (eds.) Proceedings of the 34th International
  Conference on Machine Learning ({ICML} 2017). Proceedings of Machine Learning
  Research, vol.~70, pp. 22--31. {PMLR} (2017)

\bibitem{AkazakiCPSfalsification}
Akazaki, T., Liu, S., Yamagata, Y., Duan, Y., Hao, J.: Falsification of
  cyber-physical systems using deep reinforcement learning. In: Havelund, K.,
  Peleska, J., Roscoe, B., de~Vink, E. (eds.) Formal Methods. pp. 456--465.
  Springer International Publishing, Cham (2018)

\bibitem{DBLP:conf/aaai/AlshiekhBEKNT18}
Alshiekh, M., Bloem, R., Ehlers, R., K{\"{o}}nighofer, B., Niekum, S., Topcu,
  U.: Safe reinforcement learning via shielding. In: McIlraith, S.A.,
  Weinberger, K.Q. (eds.) Proceedings of the Thirty-Second {AAAI} Conference on
  Artificial Intelligence ({AAAI} 2018). {AAAI} Press (2018)

\bibitem{DBLP:conf/hybrid/AlurCHH92}
Alur, R., Courcoubetis, C., Henzinger, T.A., Ho, P.: Hybrid automata: An
  algorithmic approach to the specification and verification of hybrid systems.
  In: Grossman, R.L., Nerode, A., Ravn, A.P., Rischel, H. (eds.) Hybrid
  Systems. LNCS, vol.~736, pp. 209--229. Springer (1992)

\bibitem{AsarinSynthesis}
Asarin, E., Bournez, O., Dang, T., Maler, O., Pnueli, A.: Effective synthesis
  of switching controllers for linear systems. Proceedings of the IEEE
  \textbf{88}(7),  1011--1025 (July 2000)

\bibitem{DBLP:journals/deds/BartoM03}
Barto, A.G., Mahadevan, S.: Recent advances in hierarchical reinforcement
  learning. Discrete Event Dynamic Systems  \textbf{13}(1-2),  41--77 (2003)

\bibitem{DBLP:conf/pldi/BohrerTMMP18}
Bohrer, B., Tan, Y.K., Mitsch, S., Myreen, M.O., Platzer, A.: {VeriPhy}:
  Verified controller executables from verified cyber-physical system models.
  In: Grossman, D. (ed.) Proceedings of the 39th {ACM} {SIGPLAN} Conference on
  Programming Language Design and Implementation ({PLDI} 2018). pp. 617--630.
  {ACM} (2018)

\bibitem{DBLP:conf/atva/BrazdilCCFKKPU14}
Br{\'{a}}zdil, T., Chatterjee, K., Chmelik, M., Forejt, V.,
  Kret{\'{\i}}nsk{\'{y}}, J., Kwiatkowska, M.Z., Parker, D., Ujma, M.:
  Verification of markov decision processes using learning algorithms. In:
  Automated Technology for Verification and Analysis - 12th International
  Symposium ({ATVA} 2014). pp. 98--114 (2014)

\bibitem{DBLP:conf/icra/FridovichKeilH18}
Fridovich{-}Keil, D., Herbert, S.L., Fisac, J.F., Deglurkar, S., Tomlin, C.J.:
  Planning, fast and slow: {A} framework for adaptive real-time safe trajectory
  planning. In: {IEEE} International Conference on Robotics and Automation
  ({ICRA}). pp. 387--394 (2018)

\bibitem{FultonThesis}
Fulton, N.: Verifiably Safe Autonomy for Cyber-Physical Systems. Ph.D. thesis,
  Computer Science Department, School of Computer Science, Carnegie Mellon
  University (2018)

\bibitem{bellerophon}
Fulton, N., Mitsch, S., Bohrer, B., Platzer, A.: Bellerophon: Tactical theorem
  proving for hybrid systems. In: Ayala{-}Rinc{\'{o}}n, M., Mu{\~{n}}oz, C.A.
  (eds.) Interactive Theorem Proving - 8th International Conference ({ITP}
  2017). {LNCS}, vol. 10499, pp. 207--224. Springer (2017)

\bibitem{DBLP:conf/cade/FultonMQVP15}
Fulton, N., Mitsch, S., Quesel, J.D., V{\"o}lp, M., Platzer, A.: {KeYmaera X}:
  An axiomatic tactical theorem prover for hybrid systems. In: Felty, A.P.,
  Middeldorp, A. (eds.) CADE. LNCS, vol.~9195, pp. 527--538. Springer (2015)

\bibitem{itc18}
Fulton, N., Platzer, A.: {S}afe {AI} for {CPS} (invited paper). In: {IEEE}
  International Test Conference ({ITC} 2018) (2018)

\bibitem{aaai18}
Fulton, N., Platzer, A.: Safe reinforcement learning via formal methods: Toward
  safe control through proof and learning. In: McIlraith, S., Weinberger, K.
  (eds.) Proceedings of the Thirty-Second {AAAI} Conference on Artificial
  Intelligence ({AAAI} 2018). pp. 6485--6492. {AAAI} Press (2018)

\bibitem{DBLP:journals/jmlr/GarciaF15}
Garc{\'{\i}}a, J., Fern{\'{a}}ndez, F.: A comprehensive survey on safe
  reinforcement learning. Journal of Machine Learning Research  \textbf{16},
  1437--1480 (2015)

\bibitem{DBLP:journals/corr/ghoshetal}
Ghosh, S., Berkenkamp, F., Ranade, G., Qadeer, S., Kapoor, A.: {Verifying
  Controllers Against Adversarial Examples with Bayesian Optimization}. CoRR
  \textbf{abs/1802.08678} (2018)

\bibitem{HasanbeigKroening}
Hasanbeig, M., Abate, A., Kroening, D.: Logically-correct reinforcement
  learning. CoRR  \textbf{abs/1801.08099} (2018)

\bibitem{DBLP:conf/qest/HenriquesMZPC12}
Henriques, D., Martins, J.G., Zuliani, P., Platzer, A., Clarke, E.M.:
  Statistical model checking for {Markov} decision processes. In: QEST. pp.
  84--93. IEEE Computer Society (2012). \doi{10.1109/QEST.2012.19}

\bibitem{DBLP:conf/cdc/HerbertCHBFT17}
Herbert, S.L., Chen, M., Han, S., Bansal, S., Fisac, J.F., Tomlin, C.J.:
  {FaSTrack}: {A} modular framework for fast and guaranteed safe motion
  planning. In: {IEEE} Annual Conference on Decision and Control ({CDC})

\bibitem{DBLP:journals/corr/bettina20182}
Jansen, N., K{\"{o}}nighofer, B., Junges, S., Bloem, R.: Shielded
  decision-making in {MDP}s. CoRR  \textbf{abs/1807.06096} (2018)

\bibitem{DBLP:conf/tacas/Junges0DTK16}
Junges, S., Jansen, N., Dehnert, C., Topcu, U., Katoen, J.: Safety-constrained
  reinforcement learning for mdps. In: Chechik, M., Raskin, J. (eds.) Tools and
  Algorithms for the Construction and Analysis of Systems - 22nd International
  Conference ({TACAS}/{ETAPS} 2016). {LNCS}, vol.~9636, pp. 130--146. Springer
  (2016)

\bibitem{RANDDriveToSafety}
Kalra, N., Paddock, S.M.: Driving to Safety: How Many Miles of Driving Would It
  Take to Demonstrate Autonomous Vehicle Reliability? RAND Corporation (2016)

\bibitem{DBLP:conf/aaip/Kitzelmann09}
Kitzelmann, E.: Inductive programming: {A} survey of program synthesis
  techniques. In: Schmid, U., Kitzelmann, E., Plasmeijer, R. (eds.) Third
  International Workshop on Approaches and Applications of Inductive
  Programming ({AAIP} 2009). Lecture Notes in Computer Science, vol.~5812, pp.
  50--73. Springer (2009)

\bibitem{DBLP:journals/tse/GouesNFW12}
{Le Goues}, C., Nguyen, T., Forrest, S., Weimer, W.: Genprog: {A} generic
  method for automatic software repair. {IEEE} Trans. Software Eng.
  \textbf{38}(1),  54--72 (2012)

\bibitem{DBLP:journals/fmsd/MitschP16}
Mitsch, S., Platzer, A.: {ModelPlex}: Verified runtime validation of verified
  cyber-physical system models. Form. Methods Syst. Des.  \textbf{49}(1),
  33--74 (2016), special issue of selected papers from RV'14

\bibitem{DBLP:journals/jar/Platzer08}
Platzer, A.: Differential dynamic logic for hybrid systems. J. Autom. Reas.
  \textbf{41}(2),  143--189 (2008)

\bibitem{DBLP:conf/lics/Platzer12a}
Platzer, A.: Logics of dynamical systems. In: LICS. pp. 13--24. IEEE (2012)

\bibitem{DBLP:journals/jar/Platzer17}
Platzer, A.: A complete uniform substitution calculus for differential dynamic
  logic. J. Autom. Reas.  \textbf{59}(2),  219--266 (2017)

\bibitem{DBLP:conf/fm/RothenbergG16}
Rothenberg, B., Grumberg, O.: Sound and complete mutation-based program repair.
  In: Fitzgerald, J.S., Heitmeyer, C.L., Gnesi, S., Philippou, A. (eds.) Formal
  Methods - 21st International Symposium ({FM} 2016). {LNCS}, vol.~9995, pp.
  593--611 (2016)

\bibitem{DBLP:conf/hybrid/SadraddiniB18}
Sadraddini, S., Belta, C.: Formal guarantees in data-driven model
  identification and control synthesis. In: Proceedings of the 21st
  International Conference on Hybrid Systems: Computation and Control ({HSCC}
  2018). pp. 147--156 (2018)

\bibitem{DBLP:conf/icml/SchulmanLAJM15}
Schulman, J., Levine, S., Abbeel, P., Jordan, M.I., Moritz, P.: Trust region
  policy optimization. In: Bach, F.R., Blei, D.M. (eds.) Proceedings of the
  32nd International Conference on Machine Learning ({ICML} 2015). {JMLR}
  Workshop and Conference Proceedings, vol.~37, pp. 1889--1897 (2015)

\bibitem{sutton.barto:reinforcement}
Sutton, R.S., Barto, A.G.: Reinforcement Learning: An Introduction. MIT Press,
  Cambridge, MA (1998)

\end{thebibliography}
\clearpage

\end{document}